\documentclass[letterpaper, 11pt]{article}
\usepackage[utf8]{inputenc}

\usepackage{geometry}
\usepackage{cite}
\usepackage{amsmath,amssymb,amsfonts}
\usepackage{algorithm}
\usepackage{graphicx}
\usepackage{textcomp}

\usepackage{amsthm}
\usepackage{tikz}
\usetikzlibrary{shapes.geometric, arrows, backgrounds, scopes, positioning}
\usepackage{subcaption}
\usepackage{multirow}
\usepackage{breqn}
\usepackage{algpseudocode}
\usepackage{hyperref}

\providecommand{\norm}[1]{\ensuremath{\left\lVert#1\right\rVert }}
\providecommand{\mnorm}[1]{\ensuremath{\left\lvert#1\right\rvert}}

\def\R{\mathbb{R}}

\newtheorem{theorem}{\bfseries Theorem}

\newtheorem{assumption}{\bfseries Assumption}

\DeclareRobustCommand{\bigO}{%
  \text{\usefont{OMS}{cmsy}{m}{n}O}%
}

\title{
Generalized AdaGrad (G-AdaGrad) and Adam:\\ A State-Space Perspective
}

\author{Kushal Chakrabarti$^1$, and Nikhil Chopra$^1$
\thanks{$^1$ University of Maryland, College Park, Maryland 20742, U.S.A.}
\thanks{Emails: {\em kchak@terpmail.umd.edu}, and {\em nchopra@umd.edu}}%
}

\date{}

\begin{document}

\maketitle
\thispagestyle{empty}
\pagestyle{empty}


\begin{abstract}
    Accelerated gradient-based methods are being extensively used for solving non-convex machine learning problems, especially when the data points are abundant or the available data is distributed across several agents. Two of the prominent accelerated gradient algorithms are AdaGrad and Adam. AdaGrad is the simplest accelerated gradient method, which is particularly effective for sparse data. Adam has been shown to perform favorably in deep learning problems compared to other methods. In this paper, we propose a new fast optimizer, Generalized AdaGrad (G-AdaGrad), for accelerating the solution of potentially non-convex machine learning problems. Specifically, we adopt a state-space perspective for analyzing the convergence of gradient acceleration algorithms, namely G-AdaGrad and Adam, in machine learning. Our proposed state-space models are governed by ordinary differential equations. We present simple convergence proofs of these two algorithms in the deterministic settings with minimal assumptions. Our analysis also provides intuition behind improving upon AdaGrad's convergence rate. We provide empirical results on \textit{MNIST} dataset to reinforce our claims on the convergence and performance of G-AdaGrad and Adam.
\end{abstract}
\section{Introduction}
\label{sec:intro}

In this paper, we consider the minimization problem
\begin{align}
    \min_{x \in \R^d} f(x), \label{eqn:opt_1}
\end{align}
where the {\em objective function} $f: \R^d \to \R^+$ is smooth and possibly non-convex. In machine learning, $f$ is typically approximated by the average of a large number of loss functions, each loss function being associated with individual training examples or mini-batches, and $x$ typically represents the unknown weights of a model. The goal is to find a {\em critical point} of the aggregate loss function $f$.

Several first-order iterative methods exist for solving the optimization problem~\eqref{eqn:opt_1}. First-order methods are preferred when the available data-size is large~\cite{bottou2018optimization}. Besides expensive computations, another drawback of second-order methods, such as Newton's method~\cite{kelley1999iterative}, is that for linear models these methods cannot be implemented over a distributed network, as the agents do not share their data points with the server~\cite{chakrabarti2020iterative}.

The classical gradient-descent method is the basic prototype of first-order optimization methods~\cite{bertsekas1989parallel}. Its stochastic version, known as the stochastic gradient-descent (SGD), has become a popular method for solving machine learning problems, especially the large-scale problems~\cite{bottou2018optimization}. Several accelerated variants of SGD have been proposed since the past decade~\cite{duchi2011adaptive, kingma2014adam, zeiler2012adadelta, reddi2019convergence, dozat2016incorporating}. Two of such notable methods are the adaptive gradient-descent method (AdaGrad)~\cite{duchi2011adaptive}, and the adaptive momentum estimation method (Adam)~\cite{kingma2014adam}. Both AdaGrad and Adam methods maintain an estimate of a local minima in~\eqref{eqn:opt_1} and update it iteratively using the gradient of the {\em objective function} multiplied by an adaptive {\em learning rate}.

AdaGrad is a prominent optimization method that achieves significant performance gains compared to SGD. As the name suggests, AdaGrad adaptively updates the {\em learning rate} based on the information of all the previous gradients. Specifically~\cite{duchi2011adaptive}, for each iteration $k \in \{0,1,\ldots\}$, let $x_k = [x_{k,1},\ldots,x_{k,d}]^T$ denote the estimate of a local minima in~\eqref{eqn:opt_1} maintained by AdaGrad. In addition, AdaGrad maintains a set of real valued scalar parameters denoted by $\{b_{k,i}: i=1,\ldots,d\}$. The algorithm is initialized with arbitrarily chosen initial estimate $x_0 \in \R^d$ and $\{b_{k,i} > 0: i=1,\ldots,d\}$. Let the gradient of the {\em objective function} evaluated at $x \in \R^d$ be denoted as $\nabla f(x) \in \R^d$, and its $i$-th element in be denoted as $\nabla_i f(x)$ for each dimension $i \in  \{1,\ldots,d\}$. At each iteration $k \in \{0,1,\ldots\}$, each of the parameters $\{b_{k,i}: i=1,\ldots,d\}$ are updated according to $b_{k+1,i}^2 = b_{k,i}^2 + \norm{\nabla_i f(x_k)}^2$. At the same iteration $k$, the estimate is updated to $x_{k+1,i} = x_{k,i} - \eta \frac{\nabla_i f(x_k)}{b_{k+1,i}}$ for each $i \in  \{1,\ldots,d\}$. The real valued scalar parameter $\eta > 0$ is called the {\em step-size}. Thus, the {\em learning rate} in AdaGrad is adaptively weighted along each dimension by the sum of squares of the past gradients. AdaGrad has been shown to particularly effective for sparse gradients~\cite{NIPS2013_2812e5cf}, but has under-performed for some applications~\cite{wilson2017marginal}.

The Adam algorithm has been observed to compare favorably with other optimization methods for a wide range of optimization problems, including deep learning~\cite{radford2015unsupervised, peters2018deep, wu2016google}. Like AdaGrad, Adam also updates the {\em learning rate} based on the information of past gradients. However, unlike AdaGrad, Adam effectively updates the {\em learning rate} based on only a moving window of the past gradients. Specifically~\cite{kingma2014adam}, Adam maintains two sets of $d$-dimensional vectors, respectively denoted by $\mu_k = [\mu_{k,1},\ldots,\mu_{k,d}]^T$ and $v_k = [v_{k,1},\ldots,v_{k,d}]^T$. $\mu_k$ and $v_k$ are respectively known as the biased first moment estimate and biased second raw moment estimate. These vectors are initialized with $\mu_0 = 0_d$ and $\{v_{k,i} > 0: i=1,\ldots,d\}$. Three parameters $\eta > 0$, $\beta_1 \in [0,1)$, and $\beta_2 \in [0,1)$ are chosen before the iterations begin. At each iteration $k \in \{0,1,\ldots\}$, the vectors $\mu_k$ and $v_k$ are updated according to $\mu_{k+1,i} = \beta_1 \mu_{k,i} + (1-\beta_1) \nabla_i f(x_k)$ and $v_{k+1,i} = \beta_2 v_{k,i} + (1-\beta_2) \norm{\nabla_i f(x_k)}^2$ along each dimension $i \in  \{1,\ldots,d\}$. Next, the estimate $x_k$ is updated to $x_{k+1,i} = x_{k,i} - \eta \frac{\sqrt{1-\beta_2^k}}{1-\beta_1^k} \frac{\mu_{k+1,i}}{\sqrt{v_{k+1,i}}}$ for each $i \in  \{1,\ldots,d\}$. The factor $\frac{\sqrt{1-\beta_2^k}}{1-\beta_1^k}$ is responsible for the initial bias correction, as proposed in the original Adam algorithm~\cite{kingma2014adam}.  Thus, the learning rate in Adam is weighted by the exponentially moving averages of the past gradients.

Several algorithms have been proposed to improve upon the convergence of the Adam method, such as AdaShift~\cite{zhou2018adashift}, Nadam~\cite{dozat2016incorporating}, AdaMax~\cite{kingma2014adam}. Although these algorithms have demonstrated good performance in practice, they do not have theoretical convergence guarantees. While AMSGrad has been shown to perform better than Adam on CIFAR-10 dataset~\cite{reddi2019convergence}, other experiments suggest AMSGrad be similar or worse than Adam. The recently proposed AdaBelief~\cite{zhuang2020adabelief} is another variation of Adam with a theoretical convergence guarantee. Note that the RMSprop method is a special case of Adam with the parameter $\beta_1=0$~\cite{reddi2018adaptive}.

We aim to present simplified proofs of convergence of the AdaGrad and Adam algorithms to a {\em critical point} for non-convex {\em objective functions} in the deterministic settings. The first convergence guarantee of a generalized AdaGrad method for non-convex functions was proved recently in~\cite{li2019convergence}, where the additional parameter $\epsilon \geq 0$ generalizes the AdaGrad method. However, the parameter $\epsilon$ in~\cite{li2019convergence} has been assumed to be strictly positive for the convergence guarantee, which excludes the case of the original AdaGrad method~\cite{duchi2011adaptive} where $\epsilon = 0$. We first propose a more general AdaGrad model, coined G-AdaGrad, that subsumes the work in~\cite{li2019convergence}.  Our model and corresponding convergence proof allow the parameter $\epsilon$ to be negative, as well as the case of the original AdaGrad. Besides, our proof provides intuition behind how this generalization of AdaGrad impacts its convergence. The analysis for AdaGrad in~\cite{defossez2020simple} assumes the gradients to be uniformly bounded. We do not make such an assumption. Other works also analyze the convergence of AdaGrad-like algorithms for non-convex objective functions, notable among them being WNGrad~\cite{wu2018wngrad} and AdaGrad-Norm~\cite{ward2019adagrad}. 
Note that all of the aforementioned analyses of AdaGrad and AdaGrad-like algorithms are in discrete-time. We analyze AdaGrad in the continuous-time domain.

Previous works that demonstrate convergence of the Adam algorithm for non-convex {\em objective functions} include~\cite{reddi2018adaptive, de2018convergence, tong2019calibrating, barakat2020convergence, chen2018convergence, barakat2021convergence}. In~\cite{reddi2018adaptive}, the proof for Adam is provided when the algorithm parameter $\beta_1 = 0$. We consider the general parameter settings where $\beta_1 \geq 0$. An Adam-like algorithm has been proposed and analyzed in~\cite{defossez2020simple}. The proofs in~\cite{reddi2018adaptive, de2018convergence, tong2019calibrating, barakat2020convergence, chen2018convergence} do not consider the initial bias correction steps in the original Adam~\cite{kingma2014adam}. Our analysis of Adam considers the bias correction steps. The analyses in~\cite{defossez2020simple, reddi2018adaptive, de2018convergence, tong2019calibrating, chen2018convergence}
assume uniformly bounded gradients. We do not make such an assumption. The aforementioned analyses of Adam are in discrete-time. A continuous-time version of Adam has been proposed in~\cite{barakat2021convergence}, which includes the bias correction steps. However, compared to the convergence proof in~\cite{barakat2021convergence}, our proof for Adam is simpler. In addition,~\cite{barakat2021convergence} assumes that the parameters $\beta_1$ and $\beta_2$ in the Adam algorithm are functions of the {\em step-size} $\eta$ such that $\beta_1$ and $\beta_2$ tends to one as the {\em step-size} $\eta \to 0$. We do not make such an assumption in our analysis.

\subsection{Summary of Our Contributions}
\label{sub:contri}

\begin{itemize}
    \item In this paper, we first propose a more general AdaGrad algorithm, which we refer to as Generalized AdaGrad (G-AdaGrad). The proposed optimizer improves upon the convergence rate of the original AdaGrad algorithm. The original AdaGrad, discussed in Section~\ref{sec:intro}, is a special case of the proposed G-AdaGrad algorithm.
    \item We propose two state-space models, each for the G-AdaGrad algorithm and the original Adam algorithm, in continuous time-domain. The proposed state-space models are an autonomous and non-autonomous system of ordinary differential equations, respectively, for G-AdaGrad and Adam. The non-autonomy of the model for Adam is due to initial bias correction steps.
    \item Using a simple analysis of the proposed state-space models, we prove the convergence of the G-AdaGrad and the Adam algorithm to a {\em critical point} of the possibly non-convex optimization problem~\eqref{eqn:opt_1} in the deterministic settings. Our analysis requires minimal assumptions about the optimization problem~\eqref{eqn:opt_1}.
\end{itemize}
 
Compared to the existing works that analyze the convergence of the AdaGrad or the Adam algorithm for non-convex {\em objective functions}, the \textbf{major contributions} of our presented analysis are as follows.

\begin{enumerate}
    \item The gradient $\nabla f$ need not be uniformly bounded, unlike~\cite{defossez2020simple, reddi2018adaptive, de2018convergence, tong2019calibrating, chen2018convergence}.
    \item Includes the original AdaGrad algorithm and a more generalized version with intuition behind the generalization, compared to~\cite{li2019convergence}.
    \item Explanation for the choice of exponent of $b_k$ in the estimation update of AdaGrad, unlike~\cite{li2019convergence}.
    \item Includes initial bias correction steps in Adam, unlike~\cite{reddi2018adaptive, de2018convergence, tong2019calibrating, barakat2020convergence, chen2018convergence}.
    \item Analysis in continuous-time domain, unlike~\cite{li2019convergence, defossez2020simple, reddi2018adaptive, de2018convergence, tong2019calibrating, barakat2020convergence, chen2018convergence}.
    \item $\beta_1$ and $\beta_2$ in Adam need not be functions of the step-size $\eta$, unlike~\cite{barakat2021convergence}. 
    \item Simple proof of convergence, compared to the continuous-time version in~\cite{barakat2021convergence}.
\end{enumerate}

\section{Continuous-Time Generalized AdaGrad}
\label{sec:adagrad}

In this section, we propose a set of autonomous ordinary differential equations. Using first-order Euler discretization, we show that the proposed set of differential equations coincides with a general version of the AdaGrad algorithm, which we refer to as the Generalized AdaGrad (G-AdaGrad). The proposed differential equations include the original AdaGrad as a special case.

We make the following assumptions in order to present our algorithms and their convergence results.

\begin{assumption} \label{assump_1}
Assume that the minimum of function $f$ exists and is finite. In other words, $\mnorm{\min_{x \in \R^d} f(x)} < \infty$.
\end{assumption}

\begin{assumption} \label{assump_2}
Assume that $f$ is twice differentiable over its domain $\R^d$ and the entries in the Hessian matrix $\nabla^2 f(x)$ are bounded above for all $x\in \R^d$.
\end{assumption}

The above assumptions about the {\em objective function} $f$ are mild and standard in the literature of gradient-based optimization. Assumption~\ref{assump_2} is equivalent to the gradient $\nabla f$ being Lipschitz continuous, which is often referred as the function $f$ being {\em smooth}~\cite{li2019convergence, reddi2018adaptive,de2018convergence}.

\subsection{Description of Generalized AdaGrad}
\label{sub:algo_adagrad}

We propose the Generalized AdaGrad (G-AdaGrad) method which is parameterized by a positive real scalar $\alpha$.
For each dimension $i \in \{1,\ldots,d\}$ and $t \geq 0$, consider the following pair of differential equations
\begin{align}
    \Dot{x}_{ci}(t) & = \norm{\nabla_i f(x(t))}^2, \label{eqn:xc_evol} \\
    \Dot{x}_i(t) & = -\dfrac{\nabla_i f(x(t))}{\left(x_{ci}(t)\right)^\alpha},  \label{eqn:x_evol}
\end{align}
with initial conditions $x_c(0) \in \R^d$ and $x(0) \in \R^d$. We assume that the initial condition $\{x_{ci}(0) > 0 : i = 1,\ldots,d\}$.  The variable ${x}_{ci}, \forall i$ can be abstracted as dynamic controller state. 

The above pair of differential equations~\eqref{eqn:xc_evol}-\eqref{eqn:x_evol} can be seen as a continuous-time variation of the following algorithm, when~\eqref{eqn:xc_evol}-\eqref{eqn:x_evol} are discretized with a fixed sampling time $\delta > 0$. For each $i \in \{1,\ldots,d\}$ and $k \in \{0,1,\ldots\}$,
\begin{align}
    x_{ci}((k+1)\delta) & = x_{ci}(k\delta) + \delta \norm{\nabla_i f(x(k\delta))}^2, \label{eqn:xc_dscrt} \\
    x_{i}((k+1)\delta) & = x_{i}(k\delta) - \delta \dfrac{\nabla_i f(x(k\delta))}{\left(x_{ci}(k\delta)\right)^\alpha}.  \label{eqn:x_dscrt}
\end{align}

This fact can be seen from the following argument.
From Taylor series expansion of $x_{ci}((k+1)\delta)$ and $x_{i}((k+1)\delta)$ we obtain that,
\begin{align*}
    x_{ci}((k+1)\delta) & = x_{ci}(k\delta) + \delta \Dot{x}_{ci}(k\delta) + \bigO(\delta^2), \\
    x_{i}((k+1)\delta) & = x_{i}(k\delta) + \delta \Dot{x}_{i}(k\delta) + \bigO(\delta^2).
\end{align*}
Upon substituting from above,~\eqref{eqn:xc_dscrt}-\eqref{eqn:x_dscrt} can be rewritten as
\begin{align*}
    \delta \Dot{x}_{ci}(k\delta) + \bigO(\delta^2) & = \delta \norm{\nabla_i f(x(k\delta))}^2, \, i \in \{1,\ldots,d\}, \\
    \delta \Dot{x}_{i}(k\delta) + \bigO(\delta^2) & = - \delta \dfrac{\nabla_i f(x(k\delta))}{\left(x_{ci}(k\delta)\right)^\alpha}, \, i \in \{1,\ldots,d\}.
\end{align*}
Defining $t=k\delta$, in the limit $\delta \to 0$, the above equations coincide with~\eqref{eqn:xc_evol}-\eqref{eqn:x_evol}. 

Note that,~\eqref{eqn:xc_dscrt}-\eqref{eqn:x_dscrt} represents a generalization of the AdaGrad algorithm discussed in Section~\ref{sec:intro} with {\em step-size} $\eta = \delta$ and an additional parameter $\alpha$. The controller states $x_c(t)$ in continuous-time corresponds to the variable $b_k$ in discrete-time of the AdaGrad algorithm. When we set $\alpha=0.5$,~\eqref{eqn:xc_dscrt}-\eqref{eqn:x_dscrt} correspond to the original AdaGrad algorithm. Introducing the parameter $\alpha$ can further improve its convergence. This is discussed in the following subsection.

\subsection{Convergence of Generalized AdaGrad}
\label{sub:conv_adagrad}

Define the set of {\em critical points} of the {\em objective function} $f$ as
\begin{align}
    X^* = \{x\in \R^d : \nabla f(x) = 0_d\}. \label{eqn:xstar}
\end{align}
Theorem~\ref{thm:adagrad} below presents a key result on the convergence of the G-AdaGrad algorithm~\eqref{eqn:xc_evol}-\eqref{eqn:x_evol} in continuous-time to a {\em critical point} in $X^*$.

\begin{theorem} \label{thm:adagrad}
Consider the pair of differential equations~\eqref{eqn:xc_evol}-\eqref{eqn:x_evol} with initial conditions $x_c(0) \in \R^d$ and $x(0) \in \R^d$ such that $\{x_{ci}(0) > 0 : i = 1,\ldots,d\}$. Let the parameter $\alpha \in (0,1)$. If Assumptions~\ref{assump_1}-\ref{assump_2} hold, then
\begin{align}
    \lim_{t \to \infty} \nabla f(x(t)) = 0_d. \label{eqn:zero_grad_1}
\end{align}
Moreover, for all $t \geq 0$, we have
\begin{align}
   f(x(t)) = f(x(0)) + \sum_{i=1}^d \dfrac{\left(x_{ci}(0)\right)^{1-\alpha} - \left(x_{ci}(0) + \int_0^t \norm{\nabla_i f(x(s))}^2 ds\right)^{1-\alpha}}{1-\alpha}. \label{eqn:ft}
\end{align}
\end{theorem}

\begin{proof}
The time-derivative of $f$ along the trajectories $x(t)$  of~\eqref{eqn:x_evol} is given by
\begin{align*}
    \Dot{f}(x(t)) & = (\nabla f(x(t))^T \Dot{x}(t) = \sum_{i=1}^d \nabla_i f(x(t)) \Dot{x}_{i}(t).
\end{align*}
Substituting~\eqref{eqn:x_evol} yields,
\begin{align*}
    \Dot{f}(x(t)) & =  - \sum_{i=1}^d \dfrac{\norm{\nabla_i f(x(t))}^2}{\left(x_{ci}(t)\right)^\alpha}.
\end{align*}
Further utilizing~\eqref{eqn:xc_evol} we get,
\begin{align}
    \Dot{f}(x(t)) & =  - \sum_{i=1}^d \dfrac{\Dot{x}_{ci}(t)}{\left(x_{ci}(t)\right)^\alpha}. \label{eqn:fdot}
\end{align}
Integrating both sides above with respect to (w.r.t) $t$ from $0$ to $t$, we get
\begin{align}
   & f(x(t))-f(x(0)) = - \sum_{i=1}^d \int_0^t \dfrac{\Dot{x}_{ci}(s)}{\left(x_{ci}(s)\right)^\alpha} ds. \label{eqn:f_int}
\end{align}
Since $\alpha < 1$, upon evaluating the integral we have
\begin{align}
    f(x(t)) = f(x(0)) + \sum_{i=1}^d \dfrac{\left(x_{ci}(0)\right)^{1-\alpha} - \left(x_{ci}(t)\right)^{1-\alpha}}{1-\alpha}. \label{eqn:f_eval}
\end{align}
Integrating both sides of~\eqref{eqn:xc_evol} w.r.t $t$ from $0$ to $t$, we have
\begin{align*}
    x_{ci}(t) = x_{ci}(0) + \int_0^t \norm{\nabla_i f(x(s))}^2 ds, ~ i \in \{1,\ldots,d\}.
\end{align*}
Using the above equation in~\eqref{eqn:f_eval} proves~\eqref{eqn:ft}.

Since $x_{ci}(0) > 0$, we have $x_{ci}(t) > 0$.
The above equation implies that $x_{ci}(t)$ is non-decreasing w.r.t $t$, which combined with~\eqref{eqn:ft} and $\alpha\in (0,1)$ implies that $f(x(t))$ in non-increasing w.r.t. $t$. From Assumption~\ref{assump_1}, $f$ is bounded below. Thus, $\lim_{t \to \infty} f(x(t))$ is finite. From~\eqref{eqn:f_eval} then it follows that, $\lim_{t \to \infty} x_{c}(t)$ is finite. Thus, the above equation implies that $\norm{\nabla f(x(t))}$ is square-integrable w.r.t $t$. Hence, $\nabla f(x(t))$ is bounded above. 

Since $\nabla f(x(t))$ is bounded and $x_{ci}(t) > 0$,  from~\eqref{eqn:x_evol} we have that $\Dot{x}(t)$ is bounded above. Now, the time-derivative of $\norm{\nabla f}^2$ along the trajectories $x(t)$ is given by
\begin{align*}
    \dfrac{d}{dt} \norm{\nabla f(x(t))}^2 = 2\nabla f(x(t))^T \nabla^2 f(x(t)) \Dot{x}(t).
\end{align*}
We have shown that $\nabla f(x(t))$ and $\Dot{x}(t)$ are bounded above. From Assumption~\ref{assump_2}, we have all the entries in $\nabla^2 f(x(t))$ bounded above. Then, from the above equation we have $\dfrac{d}{dt} \norm{\nabla f(x(t))}^2$ bounded above. Thus, $\norm{\nabla f(x(t))}^2$ is uniformly continuous.

We have shown that, $\norm{\nabla f(x(t))}$ is square-integrable and $\norm{\nabla f(x(t))}^2$ is uniformly continuous. From Barbalat's lemma~\cite{barbalat1959systemes} it follows that $\lim_{t \to \infty} \norm{\nabla f(x(t))}^2 = 0$. This proves~\eqref{eqn:zero_grad_1}.

\end{proof}

Theorem~\ref{thm:adagrad} implies that the G-AdaGrad algorithm, proposed in~\eqref{eqn:xc_evol}-\eqref{eqn:x_evol}, converges to a {\em critical point} in $X^*$ of the non-convex optimization problem~\eqref{eqn:opt_1}. Furthermore,~\eqref{eqn:ft} implies that the convergence of G-AdaGrad is affected by the algorithm parameter $\alpha$. As we will show through simulations in Section~\ref{sec:exp}, $\alpha=0.5$, which corresponds to the original AdaGrad method~\cite{duchi2011adaptive}, is not the optimal value of $\alpha$. 

Another significance of the above proof is that, it explains why the exponent $\alpha$ of $x_c(t)$ (equivalently, $b_k$ in discrete-time) in the update equation of the estimate $x(t)$ is limited to $\alpha < 1$. If $\alpha > 1$,~\eqref{eqn:ft} implies that $f(x(t))$ will be increasing in $t$. If $\alpha = 1$, evaluating the integral in~\eqref{eqn:f_int} we have
\begin{align*}
   & f(x(t)) = f(x(0)) + \sum_{i=1}^d \log \left(\frac{x_{ci}(0)}{x_{ci}(t)}\right).
\end{align*}
Thus, $f$ decreases at a slower rate for $\alpha = 1$ compared to $\alpha < 1$, because of logarithmic decrements in case of $\alpha =1$ compared to exponential decrements.

Note that, the parameter $\epsilon$ in~\cite{li2019convergence} plays the same role as $\alpha$. However, the convergence results in~\cite{li2019convergence} is only for $\epsilon \in (0,0.5]$ which corresponds to $\alpha \in (0.5,1]$. Thus, our analysis is more general compared to~\cite{li2019convergence}. In addition, our analysis in Theorem~\ref{thm:adagrad} explains the significance of the parameter $\alpha$, as discussed in the previous paragraph.

\section{Continuous-Time Adam}
\label{sec:adam}

In this section, we propose a set of non-autonomous ordinary differential equations. Using first-order Euler discretization, the proposed set of differential equations coincides with the original Adam algorithm. 

\subsection{Description of Adam}
\label{sub:algo_adam}

In order to present our algorithm, define two real valued scalar parameters $\lambda_1 \in (0,1)$ and $\lambda_2 \in (0,1)$.
Define a function $\alpha: [1,\infty) \to \R$ as
\begin{align}
    \alpha(t) &= \dfrac{1-(1-\lambda_1)^t}{\sqrt{1-(1-\lambda_2)^t}}, \, t\geq 1. \label{eqn:alpha}
\end{align}
For each $i \in \{1,\ldots,d\}$ and $t \geq 1$, consider the following sets of differential equations
\begin{align}
    \Dot{\mu}_i(t) & = -\lambda_1 \mu_i(t) + \lambda_1 \nabla_i f(x(t)), \label{eqn:mu_evol} \\
    \Dot{v}_i(t) & = -\lambda_2 v_i(t) + \lambda_2 \norm{\nabla_i f(x(t))}^2, \label{eqn:v_evol} \\
    \Dot{x}_i(t) & = - \dfrac{1}{\alpha(t)} \dfrac{\mu_i(t)}{\sqrt{v_i(t)}}, \label{eqn:xm_evol}
\end{align}
with initial conditions $\mu(1) \in \R^d$, $v(1) \in \R^d$, and $x(1) \in \R^d$. We assume that the initial condition $\{v_i(1) > 0 : i = 1,\ldots,d\}$. The variables $\mu$ and $v$ can be abstracted as dynamic controller states. 

Following the argument in Section~\ref{sub:algo_adagrad}, the above pair of differential equations~\eqref{eqn:mu_evol}-\eqref{eqn:xm_evol} can be seen as a continuous-time variation of the following algorithm, when~\eqref{eqn:mu_evol}-\eqref{eqn:xm_evol} are discretized with a sampling time $\delta > 0$. For each $i \in \{1,\ldots,d\}$ and $k \in \{1,2,\ldots\}$,
\begin{align}
    \mu_i((k+1)\delta) & = (1-\delta \lambda_1) \mu_i(k\delta) + \delta \lambda_1 \nabla_i f(x(k\delta)), \label{eqn:mu_dscrt}\\
    v_i((k+1)\delta) & = (1-\delta \lambda_2) \mu_i(k\delta) + \delta \lambda_2 \norm{\nabla_i f(x(k\delta))}^2, \label{eqn:v_dscrt}\\
    x_{i}((k+1)\delta) & = x_{i}(k\delta) - \dfrac{\delta}{\alpha(k\delta)} \dfrac{\mu_i(k\delta)}{\sqrt{v_i(k\delta)}}.  \label{eqn:xm_dscrt}
\end{align}

Note that,~\eqref{eqn:mu_dscrt}-\eqref{eqn:xm_dscrt} represents the Adam algorithm discussed in Section~\ref{sec:intro}, with the parameters $\beta_1 = 1-\delta\lambda_1$, $\beta_2 = 1-\delta\lambda_2$ and {\em step-size} $\eta = \delta$. The term $\alpha(t)$ in~\eqref{eqn:xm_evol} of captures the initial bias corrections in Adam. This term renders the system of our differential equations~\eqref{eqn:mu_evol}-\eqref{eqn:xm_evol} as non-autonomous.

In the next subsection, we present the convergence of our proposed state-space model in~\eqref{eqn:mu_evol}-\eqref{eqn:xm_evol}.

\subsection{Convergence of Adam}
\label{sub:conv_adam}

Recall the definition of the set of {\em critical points} $X^*$ from~\eqref{eqn:xstar} in Section~\ref{sub:conv_adagrad}.
Theorem~\ref{thm:adam} below proves the convergence of the Adam algorithm~\eqref{eqn:mu_evol}-\eqref{eqn:xm_evol} in continuous-time domain to a {\em critical point} in $X^*$.

\begin{theorem} \label{thm:adam}
Consider the set of differential equations~\eqref{eqn:mu_evol}-\eqref{eqn:xm_evol} with initial conditions $\mu(1) = 0_d$, $v(1) \in \R^d$ and $x(1) \in \R^d$ such that $\{v_i(1) > 0 : i = 1,\ldots,d\}$. Let the parameters $\lambda_1$ and $\lambda_2$ satisfy
\begin{align}
    0 < \lambda_2 < \lambda_1 < 1. \label{eqn:lambda_cond}
\end{align}
If Assumptions~\ref{assump_1}-\ref{assump_2} hold, then $\lim_{t \to \infty} \nabla f(x(t)) = 0_d$.
\end{theorem}

\begin{proof}

The time-derivative of $f$ along the trajectories $x(t)$  of~\eqref{eqn:xm_evol} is given by
\begin{align*}
    \Dot{f}(x(t)) &= \sum_{i=1}^d \nabla_i f(x(t)) \Dot{x}_{i}(t) = - \sum_{i=1}^d \dfrac{\nabla_i f(x(t))}{\alpha(t)} \dfrac{\mu_i(t)}{\sqrt{v_i(t)}}.
\end{align*}
Multiplying with $\alpha(t)$ on both sides above we get
\begin{align*}
    \alpha(t)\Dot{f}(x(t)) = - \sum_{i=1}^d \nabla_i f(x(t)) \dfrac{\mu_i(t)}{\sqrt{v_i(t)}}.
\end{align*}
Upon integrating both sides above w.r.t. $t$ from $1$ to $t$ and substituting from~\eqref{eqn:mu_evol} we have
\begin{align}
    & \int_1^t \alpha(s)\Dot{f}(x(s))ds = - \sum_{i=1}^d \int_1^t \dfrac{(\Dot{\mu}_i(s) + \lambda_1 \mu_i(s)) \mu_i(s)}{\lambda_1 \sqrt{v_i(s)}}ds \nonumber\\
    &= - \sum_{i=1}^d \int_1^t \dfrac{\mu_i(s)\Dot{\mu}_i(s)}{\lambda_1 \sqrt{v_i(s)}}ds - \sum_{i=1}^d \int_1^t \dfrac{\mu_i(s)^2}{\sqrt{v_i(s)}}ds. \label{eqn:total_int}
\end{align}
Integrating by parts we have the first term on R.H.S. as
\begin{align*}
    & \int_1^t \dfrac{\mu_i(s)\Dot{\mu}_i(s)}{\sqrt{v_i(s)}}ds
    = \left[\dfrac{\mu_i(s)^2}{2\sqrt{v_i(s)}}\right]_1^t + \dfrac{1}{4} \int_1^t \mu_i(s)^2 v_i(s)^{-1.5} \Dot{v}_i(s) ds.
\end{align*}
Upon substituting above from~\eqref{eqn:v_evol}, and using that $\mu(1)=0_d$ we have
\begin{align*}
    \int_1^t \dfrac{\mu_i(s)\Dot{\mu}_i(s)}{\sqrt{v_i(s)}}ds
    &= \dfrac{\mu_i(t)^2}{2\sqrt{v_i(t)}} - \dfrac{\lambda_2}{4}\int_1^t \mu_i(s)^2 v_i(s)^{-0.5} ds + \dfrac{\lambda_2}{4}\int_1^t \mu_i(s)^2 v_i(s)^{-1.5} \norm{\nabla_i f(x(s))}^2 ds.
\end{align*}
Upon substituting above in~\eqref{eqn:total_int} we obtain that
\begin{align}
    \int_1^t \alpha(s)\Dot{f}(x(s))ds
    = & - \sum_{i=1}^d \dfrac{\mu_i(t)^2}{2\lambda_1\sqrt{v_i(t)}} - \sum_{i=1}^d \left(1-\dfrac{\lambda_2}{4\lambda_1}\right)\int_1^t \mu_i(s)^2 v_i(s)^{-0.5} ds \nonumber \\
    & - \sum_{i=1}^d \dfrac{\lambda_2}{4\lambda_1}\int_1^t \mu_i(s)^2 v_i(s)^{-1.5} \norm{\nabla_i f(x(s))}^2 ds. \label{eqn:int_expand}
\end{align}

We define, $\gamma_1 = 1-\lambda_1$ and $\gamma_2 = 1-\lambda_2$.
Upon differentiating both sides of~\eqref{eqn:alpha} w.r.t $t$ we get $\Dot{\alpha}(t) = \dfrac{\gamma_2^t(1-\gamma_1^t)\log{\gamma_2} - 2\gamma_1^t(1-\gamma_2^t)\log{\gamma_1}}{2(1-\gamma_2^t)^{1.5}}$.
So we have
\begin{align}
    \Dot{\alpha}(t) < 0 \iff \left(\dfrac{\gamma_2}{\gamma_1}\right)^t \dfrac{1-\gamma_1^t}{1-\gamma_2^t} > 2\dfrac{\log{\gamma_1}}{\log{\gamma_2}}. \label{eqn:cond}
\end{align}
From the condition $\lambda_1 > \lambda_2$ in~\eqref{eqn:lambda_cond}, we have $1 > \gamma_2 > \gamma_1 > 0$. Then, $\left(\dfrac{\gamma_2}{\gamma_1}\right)^t$ and $\dfrac{1-\gamma_1^t}{1-\gamma_2^t}$ are, respectively, increasing and decreasing functions of $t$. Since $1 > \gamma_2 > \gamma_1 > 0$, we have $\lim_{t \to \infty} \left(\dfrac{\gamma_2}{\gamma_1}\right)^t \to \infty$ and $\lim_{t \to \infty} \dfrac{1-\gamma_1^t}{1-\gamma_2^t} = 1$. Thus, $\left(\dfrac{\gamma_2}{\gamma_1}\right)^t \dfrac{1-\gamma_1^t}{1-\gamma_2^t}$ is an increasing function of $t$. Then, there exists $T < \infty$ such that~\eqref{eqn:cond} holds for all $t \geq T$. 
Integrating by parts we rewrite the L.H.S. in~\eqref{eqn:int_expand} as
\begin{align*}
    & \int_1^t \alpha(s)\Dot{f}(x(s))ds = \left[\alpha(s)f(x(s))\right]_1^t - \int_1^t \Dot{\alpha}(s)f(x(s))ds.
\end{align*}
Upon substituting from above in~\eqref{eqn:int_expand}, for $t \geq T$,
\begin{align}
    & \alpha(t)f(x(t)) + \sum_{i=1}^d \dfrac{\mu_i(t)^2}{2\lambda_1\sqrt{v_i(t)}} = \alpha(1)f(x(1)) + \int_1^{T} \Dot{\alpha}(s)f(x(s))ds + \int_{T}^t \Dot{\alpha}(s)f(x(s))ds \nonumber \\
    & - \sum_{i=1}^d \left(1-\dfrac{\lambda_2}{4\lambda_1}\right)\int_1^t \mu_i(s)^2 v_i(s)^{-0.5} ds - \sum_{i=1}^d \dfrac{\lambda_2}{4\lambda_1}\int_1^t \mu_i(s)^2 v_i(s)^{-1.5} \norm{\nabla_i f(x(s))}^2 ds. \label{eqn:alphaV}
\end{align}
Due to~\eqref{eqn:lambda_cond} and~\eqref{eqn:cond}, the R.H.S. in~\eqref{eqn:alphaV} is decreasing in $t \geq T$. Then, the L.H.S. in~\eqref{eqn:alphaV} is also decreasing in $t \geq T$. From~\eqref{eqn:v_evol} and $v_i(1) > 0$, $v_i(t) > 0$. Since $\mu_i(t)$ and $v_i(t)$ are continuous and $v_i(t) > 0$, $\dfrac{\mu_i(t)^2}{2\lambda_1\sqrt{v_i(t)}}$ is continuous. Also, $\alpha(t)f(x(t))$ is continuous. Thus, considering the compact interval $[1,T]$, $\alpha(T)f(x(T)) + \sum_{i=1}^d \dfrac{\mu_i(T)^2}{2\lambda_1\sqrt{v_i(T)}} =: M_T$ is finite. Since the L.H.S. in~\eqref{eqn:alphaV} is decreasing in $t \geq T$, we have the L.H.S. in~\eqref{eqn:alphaV} bounded above by $M_T$ for all $t \geq T$.

From~\eqref{eqn:alphaV} then we have that $\mu_i(t)^2 v_i(t)^{-1.5} \norm{\nabla_i f(x(t))}^2$ and $\mu_i(t)^2 v_i(t)^{-0.5}$ are integrable w.r.t. $t$ and bounded above. It implies that, $\norm{\nabla_i f(x(t))}$ is bounded unless $\mu_i(t) = 0$ or $v_i(t) = \infty$. From~\eqref{eqn:xm_evol}, either of the conditions $\mu_i(t) = 0$ and $v_i(t) = \infty$ implies that $\dot{x}_i(t) = 0$ and, hence, $\frac{d}{dt}\nabla_i f(x(t)) = 0$. Due to continuity of $\nabla_i f$ and $\norm{\nabla f(x(1))} < \infty$, we then have $\norm{\nabla_i f(x(t))}$ is bounded above for all $t$.
Integrating both sides of~\eqref{eqn:mu_evol} and~\eqref{eqn:v_evol} w.r.t $t$ from $1$ to $t$, we have for $i \in \{1,\ldots,d\}$,
\begin{align*}
    \mu_i(t) & = \lambda_1 \int_1^t e^{-\lambda_1(t-s)} \nabla_i f(x(s)) ds, \\
    v_i(t) & = \lambda_2 \int_1^t e^{-\lambda_2(t-s)} \norm{\nabla_i f(x(s))}^2 ds + v_i(1)e^{-\lambda_2t}.
\end{align*}
Since $\norm{\nabla f(x(t))}$ is bounded above and $\lambda_1,\lambda_2 > 0$, the above equations implies that $\mu_i(t)$ and $v_i(t)$ are bounded above. Moreover, $v_i(t) > 0$ as $v_i(1) > 0$. From~\eqref{eqn:xm_evol} then we have, $\Dot{x}(t)$ is bounded above. From~\eqref{eqn:mu_evol} and~\eqref{eqn:xm_evol}, $\mu_i(t) = 0$ implies that $\dot{\mu}_i(t) = \lambda_1 \nabla_i f(x(t))$ and $\dot{x}_i(t) = 0$. Thus, $\mu_i(t)$ can be zero only at isolated points $t$. Otherwise, for some $h>0$ there exists an interval $(t-h,t+h)$ such that $\mu_i(s) = 0$ for all $s \in (t-h,t+h)$. In that case, $\dot{\mu}_i(s) = 0$ for all $s \in (t-h,t+h)$. Since $\dot{\mu}_i(s) = \lambda_1 \nabla_i f(x(s))$ for all $s \in (t-h,t+h)$, we then have $\nabla_i f(x(s)) = 0$ for all $s \in (t-h,t+h)$, which proves the theorem.

We have shown above that $\mu_i(t) = 0$ only at isolated points and $v_i(t)$ is bounded above. So, $\dfrac{1}{\mu_i(t)^2 v_i(t)^{-0.5}}$ is bounded above except at isolated points.
Since $\mu_i(t)^2 v_i(t)^{-0.5}$ is integrable and $\dfrac{1}{\mu_i(t)^2 v_i(t)^{-0.5}}$ is bounded above except at isolated points, we have $\dfrac{1}{\mu_i(t)^2 v_i(t)^{-0.5}}$ is integrable. Since $v_i(t)$ is bounded above and $\dfrac{1}{\mu_i(t)^2 v_i(t)^{-0.5}}$ is integrable, we have $\dfrac{1}{\mu_i(t)^2 v_i(t)^{-1.5}}$ integrable.
Now, we apply Cauchy-Schwartz inequality on the functions $\mu_i(t) v_i(t)^{-0.75} \norm{\nabla_i f(x(t))}$ and $\dfrac{1}{\mu_i(t) v_i(t)^{-0.75}}$. Since we have $\mu_i(t)^2 v_i(t)^{-1.5} \norm{\nabla_i f(x(t))}^2$ and $\dfrac{1}{\mu_i(t)^2 v_i(t)^{-1.5}}$ integrable, the Cauchy-Schwartz inequality implies that $\norm{\nabla_i f(x(t))}$ is integrable. Since $\norm{\nabla_i f(x(t))}$ is bounded and integrable, it is also square-integrable. Thus, $\norm{\nabla f(x(t)}^2$ is integrable.

So we have shown that $\norm{\nabla f(x(t)}^2$ is integrable and $\Dot{x}(t)$ is bounded above.
Following the same argument in last two paragraphs in the proof of Theorem~\ref{thm:adagrad}, under Assumption~\ref{assump_2}, we conclude that $\lim_{t \to \infty} \nabla f(x(t)) = 0_d$.
\end{proof}

\begin{figure*}[htb!]
\centering
\begin{subfigure}{.5\textwidth}
  \begin{center}
  \includegraphics[width = \textwidth]{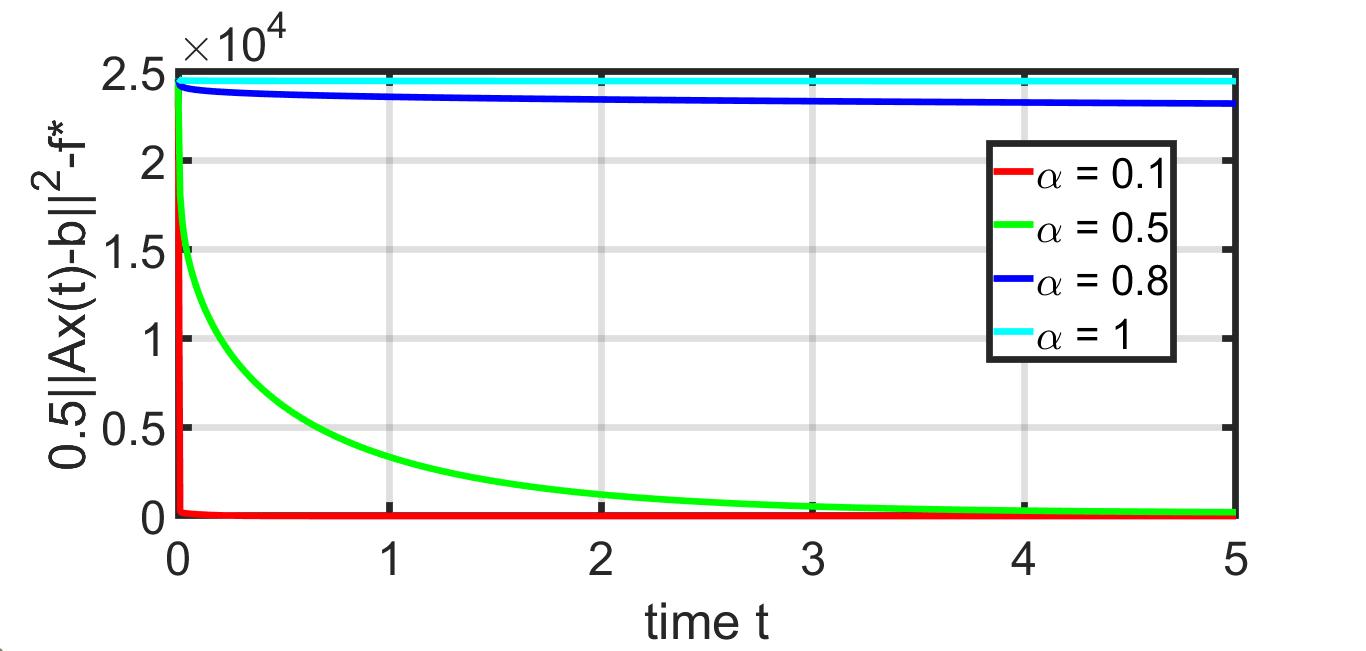}
  \caption{Generalized AdaGrad}
  \label{fig:mnist_adagrad}
  \end{center}
\end{subfigure}%
\begin{subfigure}{.5\textwidth}
  \begin{center}
  \includegraphics[width = \textwidth]{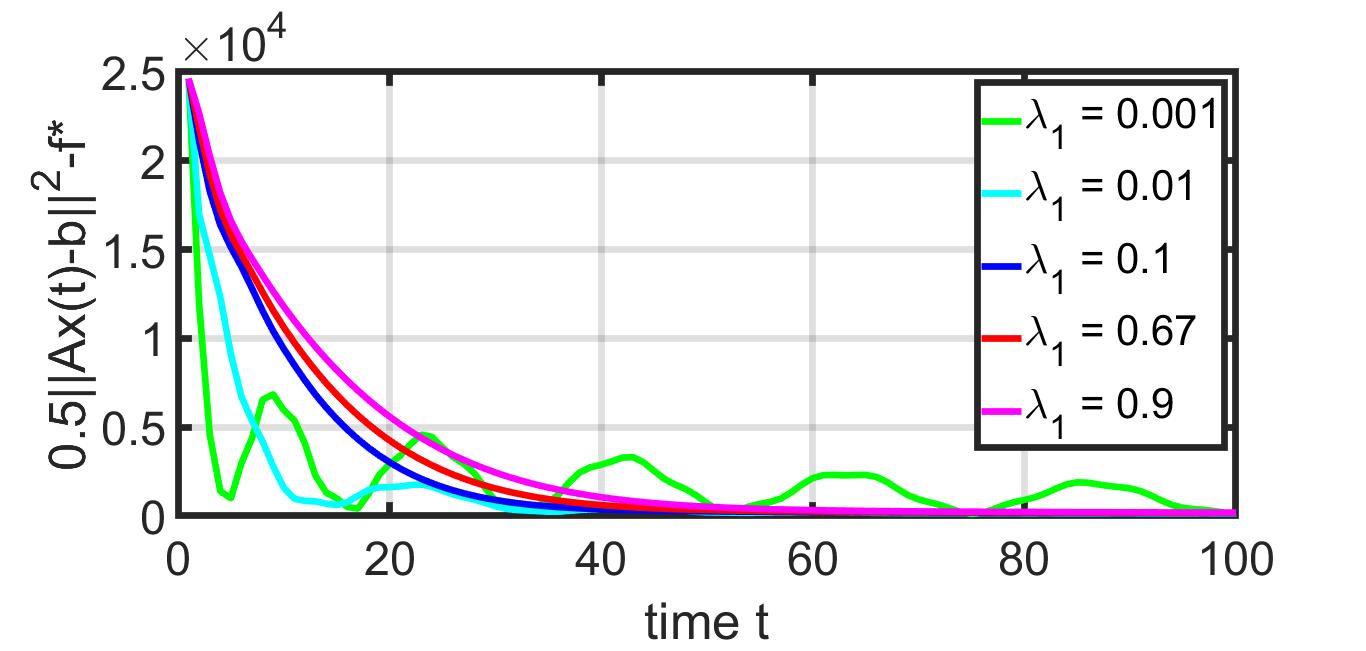}
  \caption{Adam}
  \label{fig:mnist_adam}
  \end{center}
\end{subfigure}
\caption{\footnotesize{\it Optimality gap $\frac{1}{2}\norm{Ax(t)-B}^2-f^*$ for linear regression problem with MNIST dataset, under the algorithms (a) G-AdaGrad and (b) Adam. For the G-AdaGrad algorithm, $x_c(0) = x(0) = [0.01,\ldots,0.01]^T$, and $\alpha$ is represented by different colors. For the Adam algorithm, $\mu(1) = [0,\ldots,0]^T$, $v(1) = x(1) = [0.01,\ldots,0.01]^T$, $\lambda_2 = 0.0067$, and $\lambda_1$ is represented by different colors.}}
\label{fig:mnist}
\end{figure*}

\begin{figure*}[htb!]
\centering
\begin{subfigure}{.5\textwidth}
  \begin{center}
  \includegraphics[width = \textwidth]{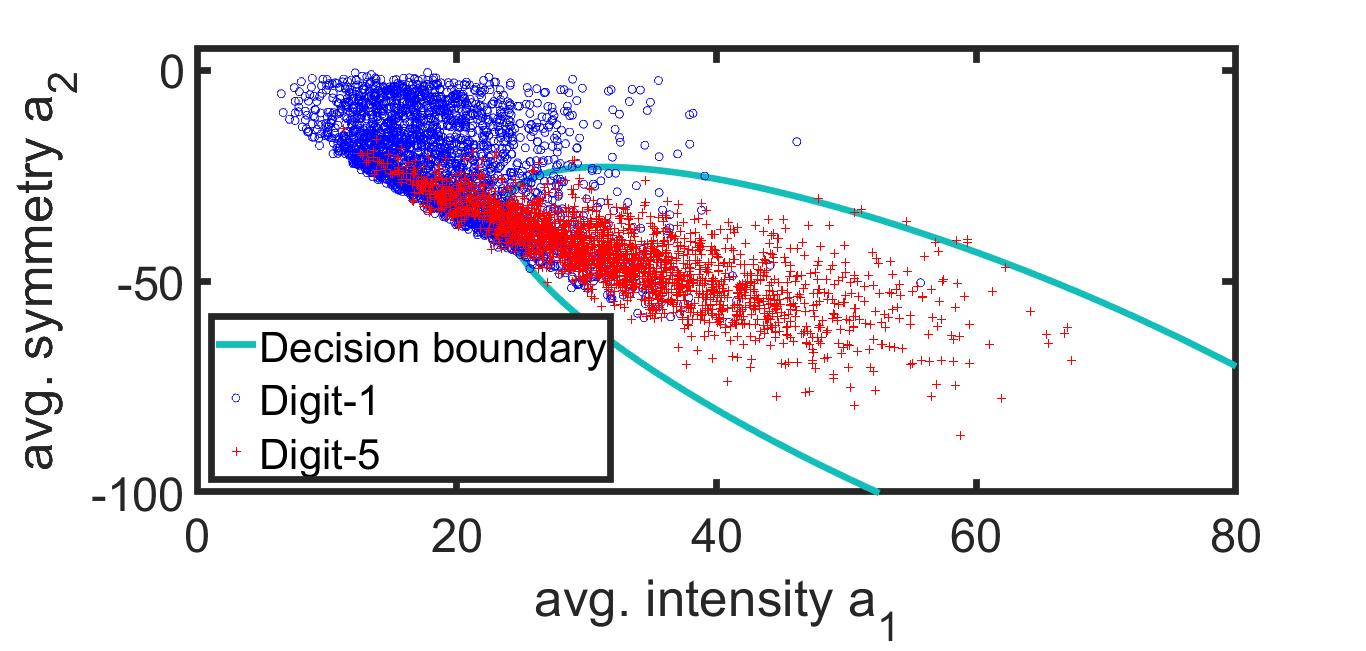}
  \caption{Training Data}
  \label{fig:train}
  \end{center}
\end{subfigure}%
\begin{subfigure}{.5\textwidth}
  \begin{center}
  \includegraphics[width = \textwidth]{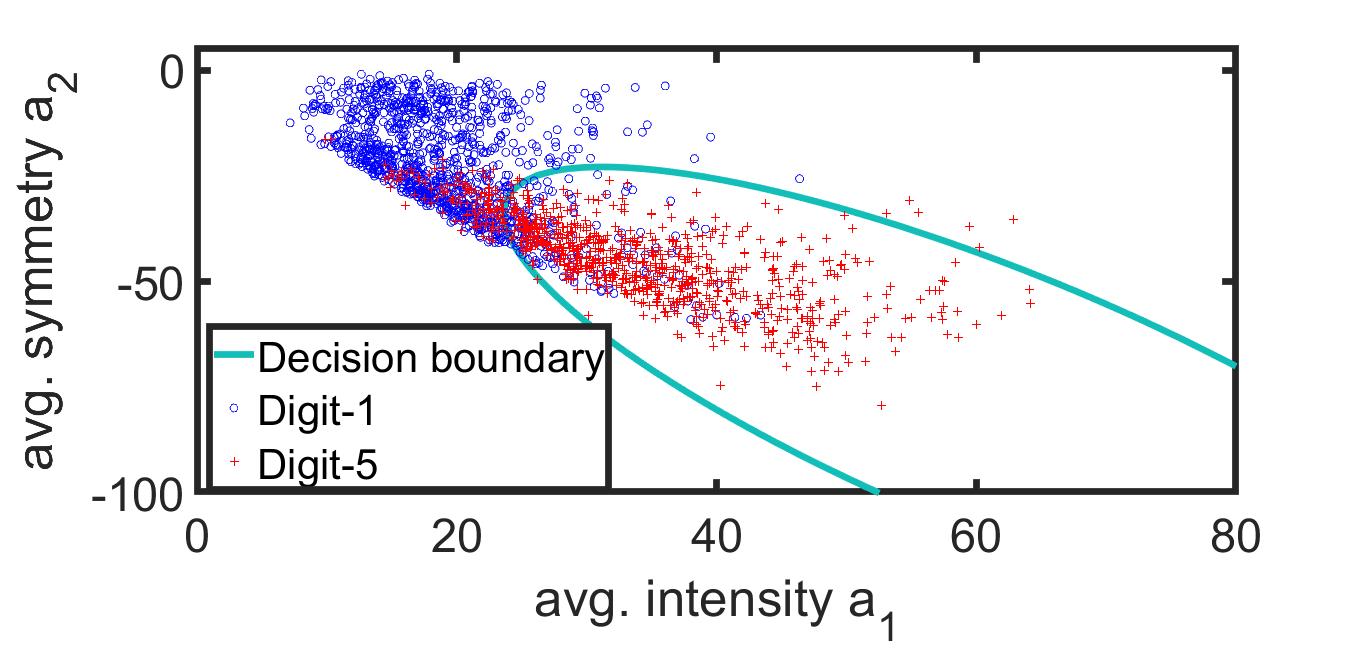}
  \caption{Test Data}
  \label{fig:test}
  \end{center}
\end{subfigure}
\caption{\footnotesize{\it Decision boundary in the $a_1-a_2$ plane, obtained from training a linear regression model for classification of digit-1 and digit-5. All the data points from (a) MNIST training set and (b) MNIST test set are plotted in $a_1-a_2$ plane. Digit-1 and digit-5 are represented by different colors.}}
\label{fig:scatter}
\end{figure*}

\begin{figure*}[htb!]
\centering
\begin{subfigure}{.5\textwidth}
  \begin{center}
  \includegraphics[width = \textwidth]{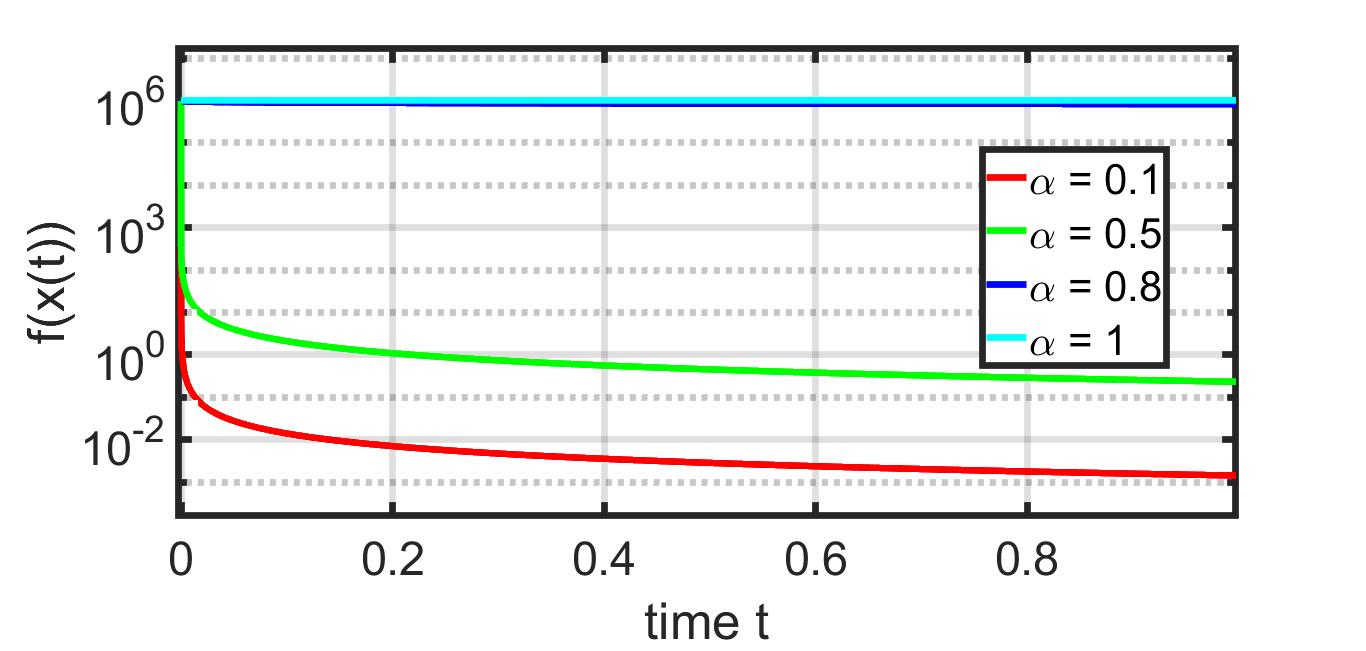}
  \caption{Generalized AdaGrad}
  \label{fig:log_adagrad}
  \end{center}
\end{subfigure}%
\begin{subfigure}{.5\textwidth}
  \begin{center}
  \includegraphics[width = \textwidth]{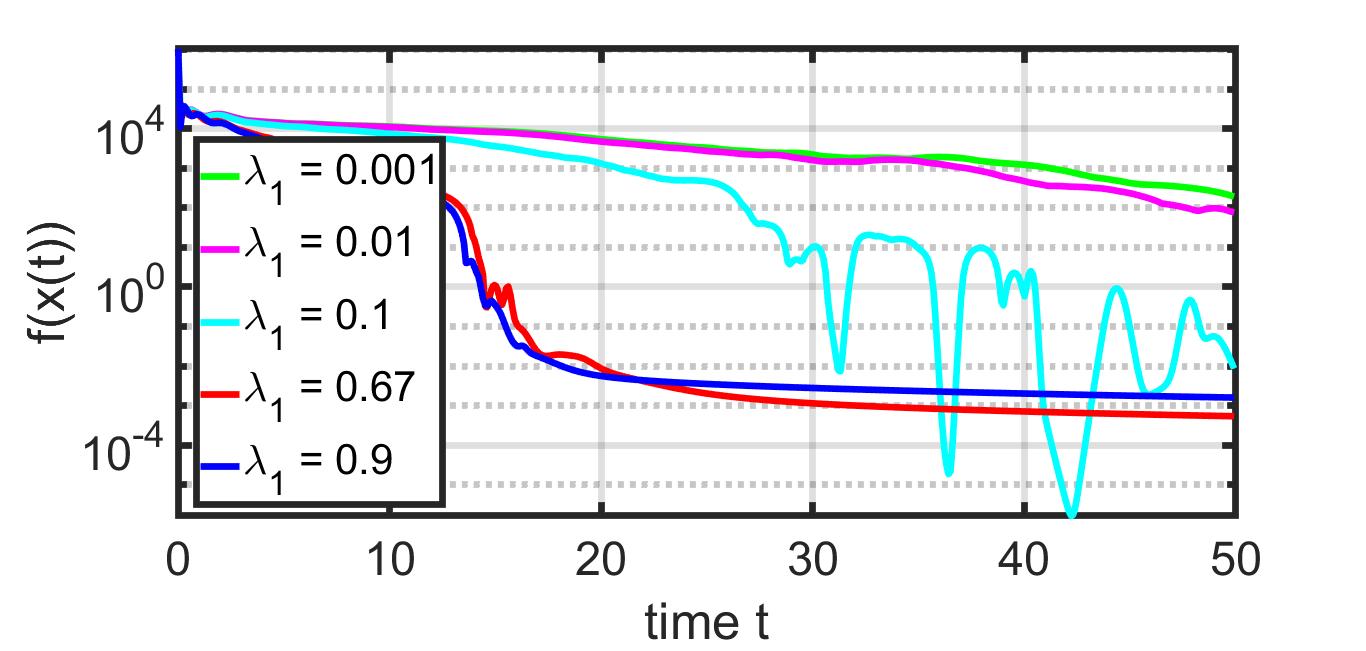}
  \caption{Adam}
  \label{fig:log_adam}
  \end{center}
\end{subfigure}
\caption{\footnotesize{\it Training loss of logistic regression model for classifying digit-1 and digit-5 from the MNIST dataset, under the algorithms (a) G-AdaGrad and (b) Adam. For the G-AdaGrad algorithm, $x_c(0) = x(0) = [0.01,\ldots,0.01]^T$, and $\alpha$ is represented by different colors. For the Adam algorithm, $\mu(1) = [0,\ldots,0]^T$, $v(1) = x(1) = [0.01,\ldots,0.01]^T$, $\lambda_2 = 0.0067$, and $\lambda_1$ is represented by different colors.}}
\label{fig:logreg}
\end{figure*}

\section{Experimental Results}
\label{sec:exp}

In this section, we present our experimental results validating the convergence guarantees from Section~\ref{sub:conv_adagrad} and Section~\ref{sub:conv_adam}. We consider the problem of recognizing handwritten digit one and digit five. 

Although it is a binary classification problem between the digits one and five, we solve it as a regression problem first. The obtained linear regression model can be a good initial decision boundary (ref. Fig.~\ref{fig:scatter}) for classification algorithms.
We conduct experiments for minimizing the {\em objective function} $f(x)=\frac{1}{2}\norm{Ax-B}^2$. The training data points $(A,B)$ are obtained from the \textit{``MNIST''}~\cite{MNIST} dataset as follows. We select $5000$ arbitrary training instances labeled as either the digit one or the digit five. For each instance, we calculate two quantities, namely the average intensity of an image and the average symmetry of an image~\cite{abu2012learning}. Let the column vectors $a_1$ and $a_2$ respectively denote the average intensity and the average symmetry of those $5000$ instances. We perform a quadratic {\em feature transform} of the data $(a_1,a_2)$. Then, our input matrix before pre-processing is $\Tilde{A} = \begin{bmatrix} a_1 & a_2 & a_1.^2 & a_1.*a_2 & a_2.^2 \end{bmatrix}$. Here, $(.*)$ represents element-wise multiplication and $(.^2)$ represents element-wise squares. This raw input matrix $\Tilde{A}$ is then pre-processed as follows. Each column of $\Tilde{A}$ is shifted by the mean value of the corresponding column and then divided by the standard deviation of that column. Finally, a $5000$-dimensional column vector of unity is appended to this pre-processed matrix. This is our final input matrix $A$ of dimension $(5000 \times 6)$. Next we consider the logistic regression model and conduct experiments for minimizing the cross-entropy error on the raw training data.

We train both of these models with the G-AdaGrad algorithm~\eqref{eqn:xc_evol}-\eqref{eqn:x_evol} and the Adam algorithm~\eqref{eqn:mu_evol}-\eqref{eqn:xm_evol}. We initialize the algorithms according to the conditions in Theorem~\ref{thm:adagrad} and Theorem~\ref{thm:adam}. Specifically, we initialize the G-AdaGrad algorithm with $x_c(0) = x(0) = [0.01,\ldots,0.01]^T$, and the Adam algorithm with $\mu(1) = [0,\ldots,0]^T$, $v(1) = x(1) = [0.01,\ldots,0.01]^T$. Moreover, we set $\lambda_2 = 0.0067$ for Adam.

G-AdaGrad converges for different values of $\alpha$ (ref. Fig.~\ref{fig:mnist_adagrad} and Fig.~\ref{fig:log_adagrad}). We observe that the convergence is faster when $\alpha$ is smaller. Thus, the coefficient $\alpha = 0.5$, which corresponds to the original AdaGrad method, is not the optimal choice. In addition, $\alpha=1$ leads to poor convergence, as we have theoretically explained in Section~\ref{sub:conv_adagrad}.

Fig.~\ref{fig:mnist_adam} and Fig.~\ref{fig:log_adam} show the effect of the relative values of $\lambda_1$ and $\lambda_2$ on the convergence of Adam algorithm. The standard choices for $\beta_1$ and $\beta_2$ in discrete-time Adam are respectively $0.9$ and $0.999$~\cite{kingma2014adam}. With a sampling time $\delta = 0.15$, from the relation between discrete-time and continuous-time Adam we have $\lambda_1 = 0.67$ and $\lambda_2 = 0.0067$ (ref. Section~\ref{sub:algo_adam}). Thus, our result in Fig.~\ref{fig:mnist_adam} and Fig.~\ref{fig:log_adam} agrees with the standard choices of these two parameters. A smaller or larger $\frac{\lambda_1}{\lambda_2}$ leads to oscillations or slows down the convergence. Note that, the condition~\eqref{eqn:lambda_cond} in Theorem~\ref{thm:adam} is satisfied with these standard parameter values. 
\section{Conclusion}

We proposed a fast optimizer, named Generalized AdaGrad or G-AdaGrad, a generalization of the prototypical AdaGrad algorithm. 
We acquired state-space frameworks of the G-AdaGrad algorithm and the Adam algorithm, governed by a set of ordinary differential equations. From the proposed state-space viewpoint, we presented simple convergence proofs of G-AdaGrad and Adam for non-convex optimization problems. Our analysis of G-AdaGrad provided further insights into the AdaGrad method. The theoretical results have been validated empirically on the \textit{MNIST} dataset. Future work involves analyzing variations of Adam, such as AdaShift~\cite{zhou2018adashift}, Nadam~\cite{dozat2016incorporating}, AdaMax~\cite{kingma2014adam}, that do not have theoretical guarantees, in the state-space framework proposed in this paper.


\bibliographystyle{unsrt}        
\bibliography{refs} 

\begin{thebibliography}{10}

\bibitem{bottou2018optimization}
L{\'e}on Bottou, Frank~E Curtis, and Jorge Nocedal.
\newblock Optimization methods for large-scale machine learning.
\newblock {\em Siam Review}, 60(2):223--311, 2018.

\bibitem{kelley1999iterative}
Carl~T Kelley.
\newblock {\em Iterative methods for optimization}.
\newblock SIAM, 1999.

\bibitem{chakrabarti2020iterative}
Kushal Chakrabarti, Nirupam Gupta, and Nikhil Chopra.
\newblock Iterative pre-conditioning for expediting the gradient-descent
  method: The distributed linear least-squares problem.
\newblock {\em arXiv preprint arXiv:2008.02856}, 2020.

\bibitem{bertsekas1989parallel}
Dimitri~P Bertsekas and John~N Tsitsiklis.
\newblock {\em Parallel and distributed computation: numerical methods},
  volume~23.
\newblock Prentice hall Englewood Cliffs, NJ, 1989.

\bibitem{duchi2011adaptive}
John Duchi, Elad Hazan, and Yoram Singer.
\newblock Adaptive subgradient methods for online learning and stochastic
  optimization.
\newblock {\em Journal of machine learning research}, 12(7), 2011.

\bibitem{kingma2014adam}
Diederik~P Kingma and Jimmy Ba.
\newblock Adam: A method for stochastic optimization.
\newblock {\em arXiv preprint arXiv:1412.6980}, 2014.

\bibitem{zeiler2012adadelta}
Matthew~D Zeiler.
\newblock Adadelta: {A}n adaptive learning rate method.
\newblock {\em arXiv preprint arXiv:1212.5701}, 2012.

\bibitem{reddi2019convergence}
Sashank~J Reddi, Satyen Kale, and Sanjiv Kumar.
\newblock On the convergence of adam and beyond.
\newblock {\em arXiv preprint arXiv:1904.09237}, 2019.

\bibitem{dozat2016incorporating}
Timothy Dozat.
\newblock Incorporating nesterov momentum into adam.
\newblock 2016.

\bibitem{NIPS2013_2812e5cf}
John Duchi, Michael~I Jordan, and Brendan McMahan.
\newblock Estimation, optimization, and parallelism when data is sparse.
\newblock In C.~J.~C. Burges, L.~Bottou, M.~Welling, Z.~Ghahramani, and K.~Q.
  Weinberger, editors, {\em Advances in Neural Information Processing Systems},
  volume~26. Curran Associates, Inc., 2013.

\bibitem{wilson2017marginal}
Ashia~C Wilson, Rebecca Roelofs, Mitchell Stern, Nathan Srebro, and Benjamin
  Recht.
\newblock The marginal value of adaptive gradient methods in machine learning.
\newblock {\em arXiv preprint arXiv:1705.08292}, 2017.

\bibitem{radford2015unsupervised}
Alec Radford, Luke Metz, and Soumith Chintala.
\newblock Unsupervised representation learning with deep convolutional
  generative adversarial networks.
\newblock {\em arXiv preprint arXiv:1511.06434}, 2015.

\bibitem{peters2018deep}
Matthew~E Peters, Mark Neumann, Mohit Iyyer, Matt Gardner, Christopher Clark,
  Kenton Lee, and Luke Zettlemoyer.
\newblock Deep contextualized word representations.
\newblock {\em arXiv preprint arXiv:1802.05365}, 2018.

\bibitem{wu2016google}
Yonghui Wu, Mike Schuster, Zhifeng Chen, Quoc~V Le, Mohammad Norouzi, Wolfgang
  Macherey, Maxim Krikun, Yuan Cao, Qin Gao, Klaus Macherey, et~al.
\newblock Google's neural machine translation system: Bridging the gap between
  human and machine translation.
\newblock {\em arXiv preprint arXiv:1609.08144}, 2016.

\bibitem{zhou2018adashift}
Zhiming Zhou, Qingru Zhang, Guansong Lu, Hongwei Wang, Weinan Zhang, and Yong
  Yu.
\newblock Adashift: Decorrelation and convergence of adaptive learning rate
  methods.
\newblock {\em arXiv preprint arXiv:1810.00143}, 2018.

\bibitem{zhuang2020adabelief}
Juntang Zhuang, Tommy Tang, Sekhar Tatikonda, Nicha Dvornek, Yifan Ding,
  Xenophon Papademetris, and James~S Duncan.
\newblock Adabelief optimizer: Adapting stepsizes by the belief in observed
  gradients.
\newblock {\em arXiv preprint arXiv:2010.07468}, 2020.

\bibitem{reddi2018adaptive}
S~Reddi, Manzil Zaheer, Devendra Sachan, Satyen Kale, and Sanjiv Kumar.
\newblock Adaptive methods for nonconvex optimization.
\newblock In {\em Proceeding of 32nd Conference on Neural Information
  Processing Systems (NIPS 2018)}, 2018.

\bibitem{li2019convergence}
Xiaoyu Li and Francesco Orabona.
\newblock On the convergence of stochastic gradient descent with adaptive
  stepsizes.
\newblock In {\em The 22nd International Conference on Artificial Intelligence
  and Statistics}, pages 983--992. PMLR, 2019.

\bibitem{defossez2020simple}
Alexandre D{\'e}fossez, L{\'e}on Bottou, Francis Bach, and Nicolas Usunier.
\newblock A simple convergence proof of {A}dam and {A}dagrad.
\newblock {\em arXiv preprint arXiv:2003.02395}, 2020.

\bibitem{wu2018wngrad}
Xiaoxia Wu, Rachel Ward, and L{\'e}on Bottou.
\newblock W{NG}rad: Learn the learning rate in gradient descent.
\newblock {\em arXiv preprint arXiv:1803.02865}, 2018.

\bibitem{ward2019adagrad}
Rachel Ward, Xiaoxia Wu, and Leon Bottou.
\newblock Adagrad stepsizes: Sharp convergence over nonconvex landscapes.
\newblock In {\em International Conference on Machine Learning}, pages
  6677--6686. PMLR, 2019.

\bibitem{de2018convergence}
Soham De, Anirbit Mukherjee, and Enayat Ullah.
\newblock Convergence guarantees for {RMSP}rop and {ADAM} in non-convex
  optimization and an empirical comparison to nesterov acceleration.
\newblock {\em arXiv preprint arXiv:1807.06766}, 2018.

\bibitem{tong2019calibrating}
Qianqian Tong, Guannan Liang, and Jinbo Bi.
\newblock Calibrating the adaptive learning rate to improve convergence of
  {ADAM}.
\newblock {\em arXiv preprint arXiv:1908.00700}, 2019.

\bibitem{barakat2020convergence}
Anas Barakat and Pascal Bianchi.
\newblock Convergence rates of a momentum algorithm with bounded adaptive step
  size for nonconvex optimization.
\newblock In {\em Asian Conference on Machine Learning}, pages 225--240. PMLR,
  2020.

\bibitem{chen2018convergence}
Xiangyi Chen, Sijia Liu, Ruoyu Sun, and Mingyi Hong.
\newblock On the convergence of a class of adam-type algorithms for non-convex
  optimization.
\newblock {\em arXiv preprint arXiv:1808.02941}, 2018.

\bibitem{barakat2021convergence}
Anas Barakat and Pascal Bianchi.
\newblock Convergence and dynamical behavior of the {ADAM} algorithm for
  nonconvex stochastic optimization.
\newblock {\em SIAM Journal on Optimization}, 31(1):244--274, 2021.

\bibitem{barbalat1959systemes}
I~Barbalat.
\newblock Systemes d’{\'e}quations diff{\'e}rentielles d’oscillations non
  lin{\'e}aires.
\newblock {\em Rev. Math. Pures Appl}, 4(2):267--270, 1959.

\bibitem{MNIST}
{MNIST} in {CSV}.
\newblock \url{https://www.kaggle.com/oddrationale/mnist-in-csv}.
\newblock Accessed: 19-September-2020.

\bibitem{abu2012learning}
Yaser~S Abu-Mostafa, Malik Magdon-Ismail, and Hsuan-Tien Lin.
\newblock {\em Learning from data}, volume~4.
\newblock AMLBook New York, NY, USA:, 2012.

\end{thebibliography}

\addtolength{\textheight}{-12cm}

\end{document}